\documentclass[11pt]{article}

\usepackage{amsmath,amssymb,amsthm}
\usepackage{graphicx}
\usepackage{booktabs}
\usepackage{hyperref}
\usepackage{url}
\usepackage[numbers,sort&compress]{natbib}
\usepackage{float}
\usepackage{placeins}
\usepackage{tikz}
\usetikzlibrary{arrows.meta, positioning, fit, backgrounds, shapes.geometric}
\usepackage{algorithm}
\usepackage{algpseudocode}
\usepackage{pifont}
\usepackage{microtype}
\usepackage{nicefrac}
\usepackage{amsfonts}
\usepackage[utf8]{inputenc}
\usepackage[T1]{fontenc}
\usepackage{lipsum}
\usepackage{geometry}
\usepackage{caption}  % برای تنظیم کپشن

\graphicspath{{media/}}
\title{RecKAN: Kolmogorov–Arnold Networks with a Learnable Recursive Polynomial Basis}

\author{
  Amirhosein Azarpour \\
  Department of Computer Science \\
  Shahid Beheshti University \\
  Tehran, Iran \\
  \texttt{am.azarpour@mail.sbu.ac.ir}
}

\begin{document}
\maketitle

\begin{abstract}
Kolmogorov--Arnold Networks (KANs) replace the fixed scalar weights of a
standard network with learnable univariate functions on each edge, but
existing variants still fix the \emph{basis} that those functions are
built from: B-splines, Chebyshev polynomials, wavelets, or Jacobi
polynomials, and learn only the combination weights over it. We
introduce RecKAN, which instead defines the basis itself by a
second order polynomial recurrence, $R_{n+1}(x) = (ax^2+bx+c)R_n(x) +
(dx+e)R_{n-1}(x)$, whose five coefficients are learned jointly with the
network. We show this recurrence recovers several classical polynomial
families including both kinds of Chebyshev polynomials, Fibonacci,
Pell, and Jacobsthal polynomials as special cases, and prove that its
degree grows linearly in $n$ exactly on the sub-family containing all of
them, giving a concrete sense in which the learned basis can move beyond
any fixed classical choice. Across multiple benchmark datasets spanning
image, text, biomedical time series classification, and time series 
forecasting, RecKAN outperforms three parameter-matched KAN baselines 
(Chebyshev, Jacobi, and spline based) on all classification tasks and 
achieves the lowest MSE on the ETTh1 forecasting benchmark. Additionally, 
when used as a classifier head with a convolutional backbone, RecKAN 
achieves higher accuracy than standard MLP heads on Fashion MNIST, 
CIFAR-10, and SVHN. On a synthetic function fitting
benchmark it tracks a sharply oscillatory target that a
parameter comparable MLP under fits. We further show that the learned
recurrence coefficients are interpretable: on the task requiring the
most local structure, training moves the basis away from the
linear degree growth regime that contains every classical family we
identify, consistent with our theoretical analysis of what that structural
shift enables.
\end{abstract}

% کلیدواژه‌ها

\newtheorem{definition}{Definition}
\newtheorem{proposition}{Proposition}
\newtheorem{remark}{Remark}

% =====================================================================
%  SECTION — INTRODUCTION
% =====================================================================
\section{Introduction}
\label{sec:introduction}

Multilayer perceptrons are universal function approximators
\citep{cybenko1989approximation, hornik1989multilayer}, but they reach
this expressiveness with a fixed architecture: every node applies a
fixed nonlinearity, and all of the network's adaptivity lives in the
scalar weights on its edges. Kolmogorov--Arnold Networks (KANs) invert
this design. Motivated by the Kolmogorov--Arnold representation theorem
\citep{kolmogorov1957}, KANs replace each scalar edge
weight with a learnable univariate function, so that adaptivity lives on
the edges rather than the nodes; the original KAN parametrizes these
edge functions with B-splines. This single change is reported to bring
two practical benefits alongside the architectural novelty: better
parameter efficiency than comparably sized MLPs on a range of function 
fitting and scientific computing tasks, and edge functions that can be
inspected and, in simple cases, read off directly as closed form
expressions an interpretability property standard MLPs largely lack
\citep{liu2024kan, liu2024kan2}.

This reframing has been followed by a rapid succession of variants that
keep the edge function idea but swap out the spline basis for something
else: Chebyshev polynomials \citep{ss2024chebykan}, wavelets
\citep{bozorgasl2024wavkan}, fractional Jacobi polynomials
\citep{aghaei2024fkan}, radial basis functions
\citep{li2024fastkan, li2024rbfkan}, and piecewise-linear ReLU-based
bases \citep{qiu2024relukan}, among others, have each been proposed as a
faster or more accurate substitute for the original spline; a recent
survey counts dozens of such variants and application-specific
descendants published within roughly a year of the original architecture
\citep{somvanshi2024survey}. The resulting architectures have been
applied to image classification \citep{azam2024kanvision}, time series
analysis \citep{vacarubio2024kantimeseries}, and scientific computing
more broadly \citep{liu2024kan2}. Despite this diversity, every one of
these variants shares a structural assumption with the original KAN: the
basis family the finite set of candidate univariate functions each
edge function is built from is fixed by the choice of architecture
before training begins, and only the linear combination weights over that
fixed basis are learned. Choosing a basis is therefore a modeling
decision made once, up front, and left to the practitioner's judgment or
to trial and error across variants; the network itself has no mechanism
to revise that choice in light of the data.

We remove that assumption. Rather than selecting a basis family in
advance, we generate it with a second order polynomial recurrence whose
five coefficients are themselves learned jointly with the rest of the
network (Section~\ref{sec:method}). As we show in
Section~\ref{subsec:special-cases}, this recurrence is expressive enough
to contain several classical polynomial families as special cases 
including both kinds of Chebyshev polynomials, Fibonacci, Pell, and Jacobsthal
polynomials precisely because those families share the property of
having recurrence coefficients that do not depend on the polynomial
degree, unlike families such as Legendre polynomials. Because the
coefficients are learned rather than fixed, training can recover one of
these known, well-understood bases when it fits the data well, or move
away from all of them into a region of function space that, to our
knowledge, no fixed basis KAN variant can reach and, as we show in
Section~\ref{subsec:exp-learned-params}, which of these two things
happens in practice depends measurably on the target task.

Our contributions are fourfold. First, we introduce a learnable
recursive polynomial basis for KANs, together with a formal
characterization of which classical polynomial families it contains as
special cases and how its expressiveness (concretely, the polynomial
degree reachable at a given recursion depth) depends on its parameters.
Second, we provide an empirical comparison against three representative
fixed or alternatively parametrized KAN baselines
\citep{ss2024chebykan, aghaei2024fkan, liu2024kan}, at matched parameter
budgets, across multiple benchmark datasets spanning image, text,
biomedical time series classification, and time series forecasting.
Third, we extend this
evaluation by replacing standard MLP classifier heads with RecKAN on top
of convolutional backbones, showing consistent improvements on three
image classification datasets. Fourth, we analyze the recurrence parameters
the model actually learns in these experiments and connect them back to
the theoretical characterization from our first contribution, giving a
concrete, interpretable account of what the model changed about its own
basis in order to fit each task, rather than treating the learned basis
as a black box.

\section{Related Work}
\label{sec:related-work}

\paragraph{Kolmogorov--Arnold Networks.}
KANs replace scalar edge weights with learnable univariate functions
\citep{liu2024kan,liu2024kan2}, motivated by the Kolmogorov--Arnold
representation theorem \citep{kolmogorov1957,arnold1957} and its earlier
critical reappraisal in the neural-network literature
\citep{girosi1989kolmogorov}. Subsequent work has studied their
theoretical expressiveness \citep{wang2024expressiveness}, benchmarked
them against MLPs and tree based models on tabular data
\citep{poeta2024tabular}, and surveyed the rapidly growing design space
\citep{somvanshi2024survey,ji2024comprehensive,seydi2024comparative}.
Architectural extensions include convolutional variants
\citep{bodner2024convkan}, graph variants
\citep{kiamari2024gkan,xu2024fourierkan}, temporal variants
\citep{genet2024tkan,vacarubio2024kantimeseries}, segmentation backbones
\citep{li2024ukan}, activation function-based reformulations
\citep{ta2025afkan}, and parameter-efficient designs
\citep{ta2025prkan,koenig2025leankan,li2024fastkan}. Applications include
computer vision \citep{azam2024kanvision}, sensor-data processing
\citep{martinezheredia2026sensorkan}, and permutation invariant function
representation \citep{shuai2025invariant}.

\paragraph{Physics-informed and scientific KANs.}
A substantial line of work embeds KANs into physics informed learning
pipelines. KINN reformulates strong, energy, and inverse-form PDE
losses on top of KAN layers and reports consistent gains over MLP-based
PINNs across multiscale, singular, and heterogeneous problems
\citep{wang2024kinn}. Related efforts include physics-informed KANs for
Navier Stokes flow \citep{shukla2025pikan}, hybrid encoder-decoder
KAN GRU architectures with augmented-Lagrangian losses for PDE solving
\citep{zhang2025alpkan}, and physics-informed KAN formulations for power
system dynamics \citep{shuai2025pikanpower}. These results are broadly
consistent with the parameter efficiency and interpretability motivations
underlying the original KAN proposal \citep{liu2024kan,liu2024kan2}.

\paragraph{Fixed and structured basis families.}
A dominant strategy replaces the original B-spline edge functions with a
different, still-fixed basis. Examples include Chebyshev polynomials
\citep{ss2024chebykan}, trainable fractional Jacobi bases
\citep{aghaei2024fkan}, wavelets \citep{bozorgasl2024wavkan}, radial-basis
formulations \citep{li2024rbfkan,li2024fastkan,dutta2026freerbfkan},
ReLU-based piecewise constructions \citep{qiu2024relukan}, and Fourier
representations \citep{xu2024fourierkan}. A broad comparative study
surveys eighteen distinct polynomial families including orthogonal,
hypergeometric, $q$-, Fibonacci-related, combinatorial, and
number theoretic polynomials as candidate KAN bases
\citep{seydi2024comparative}. Classical results on polynomial recurrences
\citep{lee2012pqfibonacci,luzon2009riordan} underlie several of these
families, since Chebyshev, Fibonacci, Pell, and Jacobsthal polynomials
all arise from constant coefficient second order recurrences. Across this
line of work, the generating family is fixed by architecture choice, and
training adjusts only coefficients, scales, or within family parameters
\citep{ss2024chebykan,aghaei2024fkan,bozorgasl2024wavkan,li2024rbfkan,
qiu2024relukan,dutta2026freerbfkan}.

\paragraph{Relation to RecKAN.}
RecKAN differs by learning the second order recurrence that generates the
basis itself, rather than selecting one fixed family in advance. This
recurrence subsumes several constant coefficient polynomial families
studied in the classical literature
\citep{lee2012pqfibonacci,luzon2009riordan} as special cases, while
remaining compatible with the broader KAN design space explored in
vision, graph, temporal, tabular, and physics informed settings
\citep{azam2024kanvision,kiamari2024gkan,genet2024tkan,poeta2024tabular,
wang2024kinn,zhang2025alpkan}. Where prior KAN variants adapt parameters
within a fixed basis family
\citep{ss2024chebykan,aghaei2024fkan,bozorgasl2024wavkan,li2024rbfkan,
qiu2024relukan,xu2024fourierkan,dutta2026freerbfkan}, RecKAN adapts the
basis-generating rule itself.

\section{Method}
\label{sec:method}
 
\subsection{Preliminaries: Kolmogorov--Arnold Networks}
\label{subsec:kan-prelim}
 
The Kolmogorov--Arnold representation theorem states that any continuous
function $f:[0,1]^n \to \mathbb{R}$ can be written as a finite superposition
of continuous univariate functions~\cite{kolmogorov1957, arnold1957}.
Kolmogorov--Arnold Networks (KANs) operationalize this result as a learning
architecture in which every edge of the network, rather than every node,
carries a learnable univariate function~\cite{liu2024kan}. Concretely, a KAN
layer mapping $\mathbb{R}^{d_{\text{in}}} \to \mathbb{R}^{d_{\text{out}}}$ is
defined edge-wise as
\begin{equation}
    y_j = \sum_{i=1}^{d_{\text{in}}} \phi_{i,j}(x_i),
    \qquad j = 1, \dots, d_{\text{out}},
    \label{eq:kan-layer-generic}
\end{equation}
where each $\phi_{i,j}:\mathbb{R}\to\mathbb{R}$ is a learnable univariate
function rather than the fixed scalar weight of a standard linear layer. The
original KAN formulation parameterizes $\phi_{i,j}$ with B-splines
\cite{liu2024kan}; subsequent work has replaced the spline basis with other
fixed function families, including Chebyshev polynomials
\cite{ss2024chebykan}, wavelets \cite{bozorgasl2024wavkan}, and fractional
Jacobi polynomials \cite{aghaei2024fkan}. In these variants, the basis family
is fixed before training, while only the linear combination coefficients over
that basis are learned. RecKAN removes this restriction by learning the
recurrence that generates the basis itself.
 
\subsection{A Learnable Recursive Polynomial Basis}
\label{subsec:recursive-basis}
 
\begin{definition}[Recursive polynomial basis]
\label{def:recurrence}
Let $a,b,c,d,e\in\mathbb{R}$. The \emph{recursive polynomial basis} of
order $K$ generated by $(a,b,c,d,e)$ is the sequence of functions
$R_0,R_1,\dots,R_K:\mathbb{R}\to\mathbb{R}$ defined by
\begin{align}
    R_0(x) &= 0, \label{eq:R0} \\
    R_1(x) &= 1, \label{eq:R1} \\
    R_{n+1}(x) &=
    \underbrace{\left(ax^2+bx+c\right)}_{\textstyle\alpha(x)}R_n(x)
    +
    \underbrace{\left(dx+e\right)}_{\textstyle\beta(x)}R_{n-1}(x),
    \qquad n=1,\dots,K-1.
    \label{eq:recurrence}
\end{align}
The five coefficients $(a,b,c,d,e)$ are shared across all recursion steps and
are treated as learnable parameters rather than fixed constants.
\end{definition}
 
Equation~\eqref{eq:recurrence} is a second order linear recurrence with
input-dependent but recursion index independent coefficients $\alpha(x)$ and
$\beta(x)$; because they are shared across all basis indices, the basis has
exactly five recurrence parameters regardless of its order $K$. This
distinguishes the construction from the general three term recurrences used
by many classical orthogonal polynomial families, whose coefficients usually
depend on the polynomial index.
 
\begin{proposition}[Degree growth]
\label{prop:degree}
For $a\neq0$, $\deg R_n\leq2(n-1)$ for $n\geq1$. For $a=0$,
$\deg R_n\leq n-1$. Equality holds in either case whenever no leading term
cancellation occurs.
\end{proposition}
\noindent
\emph{Sketch.} With $\delta_n=\deg R_n$, Eq.~\eqref{eq:recurrence} gives
$\delta_{n+1}\leq\max\{\deg\alpha+\delta_n,\,\deg\beta+\delta_{n-1}\}$; since
$\deg\alpha=2$ when $a\neq0$ (else $\leq1$) and $\deg\beta\leq1$ always, the
two stated bounds follow by induction on $\delta_1=0$. Practically: the
quadratic term $ax^2$ raises the maximum degree gained per recursion step
from one to two, so at a fixed order $K$ the $a\neq0$ regime can reach basis
elements of up to twice the degree available on the $a=0$ slice a capacity
statement only, since cancellation can keep the realized degree lower A complete proof is provided in Appendix~\ref{app:proof-degree}.
 
\subsection{Recovering Classical Polynomial Families}
\label{subsec:special-cases}
 
A typical three-term recurrence for a classical orthogonal-polynomial family
has the form
\begin{equation}
    p_{n+1}(x)=(A_nx+B_n)p_n(x)-C_np_{n-1}(x),
    \label{eq:classical-recurrence}
\end{equation}
where $A_n,B_n$, and $C_n$ may depend on $n$. The recurrence in
Definition~\ref{def:recurrence} instead uses coefficients that are constant
with respect to $n$, so only classical families with degree-independent
recurrence coefficients are reachable as special cases among them both
kinds of Chebyshev polynomials and several Fibonacci type polynomial
sequences~\cite{lee2012pqfibonacci,luzon2009riordan}.
 
\begin{proposition}[Chebyshev polynomials as a special case]
\label{prop:chebyshev}
Setting $(a,b,c,d,e)=(0,2,0,0,-1)$ in
Definition~\ref{def:recurrence} yields the recurrence
\begin{equation}
    R_{n+1}(x)=2xR_n(x)-R_{n-1}(x).
    \label{eq:chebyshev-recurrence}
\end{equation}
With the initial conditions $R_0=0$ and $R_1=1$, this recurrence produces
$R_n(x)=U_{n-1}(x)$ for $n\geq1$, where $U_n$ denotes the Chebyshev
polynomials of the second kind. Replacing the initial conditions by
$(R_0,R_1)=(1,x)$ produces the Chebyshev polynomials of the first kind,
$R_n(x)=T_n(x)$.
\end{proposition}
\noindent
Because Eq.~\eqref{eq:recurrence} is linear and second order, its solutions
for fixed $(a,b,c,d,e)$ form a two-dimensional space, so the same
coefficients can generate more than one named sequence depending on the
initial condition Chebyshev $T_n$ and $U_n$ above being the standard
example. RecKAN fixes $(R_0,R_1)=(0,1)$ throughout, which selects the
Fibonacci-type solution for whichever coefficients training settles on.
 
\begin{table}[H]
\centering
\caption{Named polynomial families recovered from the recursive basis of
Definition~\ref{def:recurrence} for selected settings of $(a,b,c,d,e)$.
``Init. match'' indicates whether the standard initial condition agrees with
$(R_0,R_1)=(0,1)$.}
\label{tab:special-cases}
\begin{tabular}{@{}lccccccl@{}}
\toprule
Family & $a$ & $b$ & $c$ & $d$ & $e$ & Init. match & Reference \\
\midrule
Chebyshev, 2nd kind $U_n$ & 0 & 2 & 0 & 0 & $-1$ & \checkmark & \cite{ss2024chebykan} \\
Chebyshev, 1st kind $T_n$ & 0 & 2 & 0 & 0 & $-1$ & $R_1=x$ & \cite{ss2024chebykan} \\
Fibonacci polynomials $F_n$ & 0 & 1 & 0 & 0 & $1$ & \checkmark & \cite{lee2012pqfibonacci} \\
Pell polynomials $p_n$ & 0 & 2 & 0 & 0 & $1$ & \checkmark & \cite{luzon2009riordan} \\
Jacobsthal polynomials $J_n$ & 0 & 0 & 1 & 2 & $0$ & \checkmark & \cite{lee2012pqfibonacci} \\
Fermat polynomials $\phi_n$ & 0 & 3 & 0 & 0 & $-2$ & $R_1=3x$ & \cite{luzon2009riordan} \\
\bottomrule
\end{tabular}
\end{table}
 
\noindent
Every listed family has $a=0$, and hence belongs to the degree one per step
sub-family described by Proposition~\ref{prop:degree}. The learned recurrence
can move beyond this slice by activating the quadratic term in $\alpha(x)$,
by introducing a linear term in $\beta(x)$, or by jointly using the remaining
coefficients bases that need not coincide with any fixed classical family.
The Chebyshev setting of Proposition~\ref{prop:chebyshev} identifies the
point at which the learnable recurrence coincides with the fixed basis used
in ChebyKAN~\cite{ss2024chebykan}, giving a natural initialization from which
optimization remains free to adapt the basis to the data.
 
\subsection{Edge Functions and the RecKAN Layer}
\label{subsec:layer}
 
Given the basis $R_0,\dots,R_K$ generated by Eq.~\eqref{eq:recurrence}, the
learnable univariate edge function in Eq.~\eqref{eq:kan-layer-generic} is
parameterized as
\begin{equation}
    \phi_{i,j}(x_i)=\sum_{n=0}^{K}w_{i,j,n}R_n(x_i),
    \label{eq:edge-function}
\end{equation}
where $w_{i,j,n}\in\mathbb{R}$ are learnable combination coefficients. For
an input $\mathbf{x}\in\mathbb{R}^{d_{\text{in}}}$, define the basis tensor
\begin{equation}
    \mathbf{R}(\mathbf{x})\in\mathbb{R}^{d_{\text{in}}\times(K+1)},
    \qquad
    \mathbf{R}(\mathbf{x})_{i,n}=R_n(x_i),
    \label{eq:basis-tensor}
\end{equation}
evaluated independently for each input coordinate. With a weight tensor
$\mathbf{W}\in\mathbb{R}^{d_{\text{in}}\times d_{\text{out}}\times(K+1)}$,
the layer output is
\begin{equation}
    y_j=\sum_{i=1}^{d_{\text{in}}}\sum_{n=0}^{K}
    \mathbf{R}(\mathbf{x})_{i,n}W_{i,j,n},
    \qquad
    \mathbf{y}=\mathbf{R}(\mathbf{x})\mathbin{\!\times_{1,3}\!}\mathbf{W}.
    \label{eq:layer-output}
\end{equation}
A RecKAN layer with $d_{\text{in}}$ inputs, $d_{\text{out}}$ outputs, and
basis order $K$ therefore contains
$d_{\text{in}}d_{\text{out}}(K+1)$ combination weights.
 
\subsection{Network Architecture and Parameter Sharing}
\label{subsec:architecture}
 
A RecKAN network of depth $L$ is a composition of RecKAN layers,
\begin{equation}
F(\mathbf{x})=
\left(\Phi^{(L)}\circ\Phi^{(L-1)}\circ\cdots\circ\Phi^{(1)}\right)(\mathbf{x}),
\label{eq:network-composition}
\end{equation}
where each $\Phi^{(\ell)}$ applies
Eqs.~\eqref{eq:basis-tensor}--\eqref{eq:layer-output} with its own
combination tensor $\mathbf{W}^{(\ell)}$. A single set of recurrence
coefficients $(a,b,c,d,e)$ is instantiated once and shared by every layer, so
the learned basis contributes exactly five shared parameters to the network,
independent of depth; only the combination weights are layer-specific.
Pre-activations entering the basis are mapped to a bounded range beforehand,
elementwise, via $h=\tanh(z)$, after which $x_i$ in
Eq.~\eqref{eq:recurrence} is replaced by $h_i$.
 
\begin{figure}[H]
\noindent\hspace*{-1.25cm} % میزان جابه جایی به چپ را با تغییر عدد تنظیم کنید
\begin{tikzpicture}[
    node distance=26mm and 24mm,
    every node/.style={font=\small},
    R/.style={draw,circle,minimum size=10mm,thick,fill=blue!6},
    param/.style={draw,rectangle,rounded corners,minimum height=8mm,
                  minimum width=34mm,fill=orange!12,thick,align=center},
    arr/.style={-{Latex[length=2.2mm]},thick},
    parr/.style={-{Latex[length=2mm]},thick,orange!70!black}
]
\node[R] (R0) {$R_0=0$};
\node[R,right=of R0] (R1) {$R_1=1$};
\node[R,right=of R1] (R2) {$R_2$};
\node[R,right=of R2] (R3) {$R_3$};
\node[R,right=of R3,fill=blue!2] (Rdots) {$\cdots$};
\node[R,right=of Rdots] (RK) {$R_K$};
 
\draw[arr] (R0.south) to[out=-55,in=-125] node[midway,below]{\scriptsize $\beta(x)$} (R2.south);
\draw[arr] (R1.south) to[out=-55,in=-125] node[midway,below]{\scriptsize $\beta(x)$} (R3.south);
\draw[arr] (R2.south) to[out=-55,in=-125] node[midway,below]{\scriptsize $\beta(x)$} (Rdots.south);
\draw[arr] (R3.south) to[out=-55,in=-125] node[midway,below]{\scriptsize $\beta(x)$} (RK.south);
\draw[arr] (R1.north) to[out=55,in=125] node[midway,above]{\scriptsize $\alpha(x)$} (R2.north);
\draw[arr] (R2.north) to[out=55,in=125] node[midway,above]{\scriptsize $\alpha(x)$} (R3.north);
\draw[arr] (R3.north) to[out=55,in=125] node[midway,above]{\scriptsize $\alpha(x)$} (Rdots.north);
\draw[arr] (Rdots.north) to[out=55,in=125] node[midway,above]{\scriptsize $\alpha(x)$} (RK.north);
 
\node[param,above=20mm of R2,xshift=12mm] (params)
{$(a,b,c,d,e)$\\shared across all steps};
\draw[parr] (params.south west) to[out=-100,in=90] (R2.north);
\draw[parr] (params.south) to[out=-90,in=90] (R3.north);
\draw[parr,dashed] (params.south east) to[out=-80,in=90] (RK.north);
\end{tikzpicture}
\caption{Computational graph of the recursive polynomial basis for a scalar
input $x$. Each basis function is computed from the two preceding functions
through the shared coefficients $(a,b,c,d,e)$.}
\label{fig:recurrence-graph}
\end{figure}
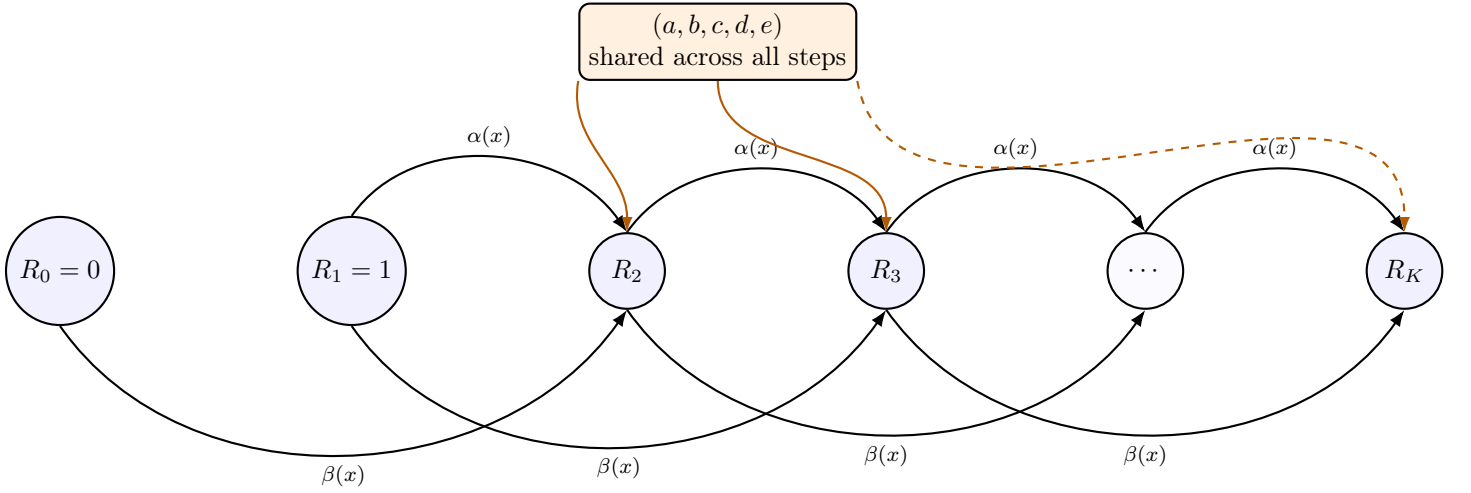
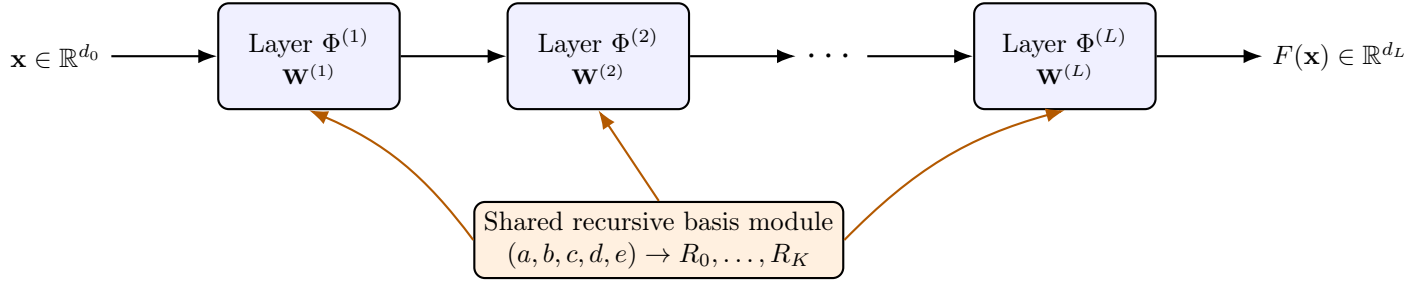
\begin{figure}[H]
\centering
\begin{tikzpicture}[
    node distance=10mm and 14mm,
    every node/.style={font=\small},
    layer/.style={draw,rectangle,rounded corners,minimum height=14mm,
                  minimum width=24mm,thick,fill=blue!6,align=center},
    basisbox/.style={draw,rectangle,rounded corners,minimum height=10mm,
                     minimum width=45mm,thick,fill=orange!12,align=center},
    io/.style={font=\small},
    arr/.style={-{Latex[length=2.5mm]},thick}
]
\node[io] (x) {$\mathbf{x}\in\mathbb{R}^{d_0}$};
\node[layer,right=of x] (l1) {Layer $\Phi^{(1)}$\\\scriptsize $\mathbf{W}^{(1)}$};
\node[layer,right=of l1] (l2) {Layer $\Phi^{(2)}$\\\scriptsize $\mathbf{W}^{(2)}$};
\node[right=of l2,font=\Large] (dots) {$\cdots$};
\node[layer,right=of dots] (lL) {Layer $\Phi^{(L)}$\\\scriptsize $\mathbf{W}^{(L)}$};
\node[io,right=of lL] (y) {$F(\mathbf{x})\in\mathbb{R}^{d_L}$};
\draw[arr] (x) -- (l1);
\draw[arr] (l1) -- (l2);
\draw[arr] (l2) -- (dots);
\draw[arr] (dots) -- (lL);
\draw[arr] (lL) -- (y);
\node[basisbox,below=12mm of l2,xshift=8mm] (basis)
{Shared recursive basis module\\$(a,b,c,d,e)\to R_0,\dots,R_K$};
\draw[arr,orange!70!black] (basis.west) to[bend right=15] (l1.south);
\draw[arr,orange!70!black] (basis.north) -- (l2.south);
\draw[arr,orange!70!black] (basis.east) to[bend left=15] (lL.south);
\end{tikzpicture}
\caption{RecKAN architecture. All layers share one recursive basis generated
by $(a,b,c,d,e)$, while each layer has separate combination weights.}
\label{fig:architecture}
\end{figure}
 
\subsection{Numerical Stability and Training Scheme}
\label{subsec:stability}
 
The recurrence coefficients can in principle take arbitrary real values, and
repeated multiplication through the recurrence can increase the magnitude of
basis functions as the basis order grows. We use four complementary
controls.
\vspace{1cm} 
\paragraph{Bounded recurrence coefficients.}
Rather than optimizing $(a,b,c,d,e)$ directly, we optimize unconstrained
parameters $(\tilde a,\tilde b,\tilde c,\tilde d,\tilde e)\in\mathbb{R}^5$
and obtain the recurrence coefficients through
\begin{equation}
\begin{aligned}
a&=B\tanh(\tilde a), & b&=B\tanh(\tilde b), & c&=B\tanh(\tilde c),\\
d&=B\tanh(\tilde d), & e&=B\tanh(\tilde e),
\end{aligned}
\label{eq:param-squash}
\end{equation}
for a fixed $B>0$. This ensures $(a,b,c,d,e)\in[-B,B]^5$ while preserving
end-to-end differentiability. When $B>2$, the exact Chebyshev point from
Proposition~\ref{prop:chebyshev} lies in the interior of the feasible region.
 
\paragraph{Magnitude normalization.}
Each newly generated basis value is rescaled by one scalar computed over the
current batch before it is used in the following recurrence step:
\begin{equation}
\widehat R_{n+1}(x)=
\frac{R_{n+1}(x)}
{\max\!\left(\varepsilon,\max_x\left|R_{n+1}(x)\right|\right)}.
\label{eq:normalization}
\end{equation}
Here $\varepsilon>0$ provides numerical safety, and the denominator is
detached from the gradient computation, so normalization acts only as a
forward pass stabilizer. A single normalizing scalar is used per basis index
$n$, not per feature, so relative scale between input coordinates is
preserved.
 
\paragraph{Input range control.}
Pre-activations are passed through $\tanh(\cdot)$ before evaluating the
recurrence, keeping recursion arguments in a bounded interval and preventing
uncontrolled growth from large inputs.
 
\paragraph{Two-timescale optimization and gradient clipping.}
The recurrence parameters affect every generated basis function, whereas the
combination weights enter linearly at the layer output. We therefore allow a
separate learning rate for the recurrence parameters,
\begin{equation}
\eta_{\text{basis}}=\gamma\eta_{\text{weights}},
\qquad \gamma\in(0,1],
\label{eq:two-lr}
\end{equation}
and optionally freeze the recurrence parameters during an initial warm-up
period, in addition to global gradient-norm clipping at every optimization
step.
 \begin{algorithm}[H]
\caption{Forward pass of a RecKAN layer}
\label{alg:forward}
\begin{algorithmic}[1]
\Require input $\mathbf{x} \in \mathbb{R}^{d_{\text{in}}}$; raw recurrence
parameters $(\tilde a,\tilde b,\tilde c,\tilde d,\tilde e)$; weight tensor
$\mathbf{W} \in \mathbb{R}^{d_{\text{in}} \times d_{\text{out}} \times (K+1)}$;
bound $B$; basis order $K$
\State $\mathbf{h} \leftarrow \tanh(\mathbf{x})$ \Comment{input range control}
\State $(a,b,c,d,e) \leftarrow B\tanh(\tilde a), B\tanh(\tilde b), B\tanh(\tilde c), B\tanh(\tilde d), B\tanh(\tilde e)$
\State $R_0 \leftarrow \mathbf{0}$, \quad $R_1 \leftarrow \mathbf{1}$ \Comment{elementwise over $\mathbf{h}$}
\For{$n = 1, \dots, K-1$}
    \State $\alpha \leftarrow a\,\mathbf{h}^2 + b\,\mathbf{h} + c$; \quad
           $\beta \leftarrow d\,\mathbf{h} + e$
    \State $R_{n+1} \leftarrow \alpha \odot R_n + \beta \odot R_{n-1}$
    \State $R_{n+1} \leftarrow R_{n+1} \,/\, \max(\varepsilon,\, \max|R_{n+1}|)$ \Comment{Eq.~\eqref{eq:normalization}, detached}
\EndFor
\State $\mathbf{R} \leftarrow \operatorname{stack}(R_0, \dots, R_K)$ \Comment{shape $d_{\text{in}} \times (K+1)$}
\State $\mathbf{y} \leftarrow \mathbf{R} \times_{1,3} \mathbf{W}$ \Comment{Eq.~\eqref{eq:layer-output}}
\State \Return $\mathbf{y}$
\end{algorithmic}
\end{algorithm}
\section{Experiments and Results}
\label{sec:experiments}

\subsection{Experimental Setup}
\label{subsec:exp-setup}

We evaluate RecKAN against three fixed or alternatively parameterized KAN
baselines: ChebyKAN~\cite{ss2024chebykan}, a Jacobi-polynomial variant in the
style of fKAN~\cite{aghaei2024fkan}, and a B-spline variant in the style of
the original KAN~\cite{liu2024kan}. We also compare RecKAN against a standard
multilayer perceptron (MLP) on the synthetic function-approximation tasks.

Every comparison uses a matched parameter protocol. For a given task, all KAN
variants use the same input hidden output shape. The basis order for
polynomial methods, or the spline grid size for the spline baseline, is chosen
so that total parameter counts are closely matched across the compared KAN
models. The synthetic MLP comparisons use parameter comparable MLPs; the exact
parameter counts are reported with the corresponding results. All models use
Adam optimization and gradient clipping. Within each task, models are trained
with the same epoch budget and no baseline is assigned a reduced optimization
budget.

\FloatBarrier
\subsection{Function Approximation}
\label{subsec:exp-function-fit}

We first evaluate the learned basis on two synthetic regression targets with
local, non stationary, and multi scale structure. These tasks are designed to
make differences in the representational behavior of a fixed basis and a
learned recurrence directly visible.

\paragraph{1D piecewise oscillatory function.}
The first target is a noisy one dimensional function containing a
high frequency oscillatory region together with a smooth, slowly varying
region. RecKAN contains 149 parameters and is compared against an MLP with 385
parameters. Figure~\ref{fig:piecewise-1d} shows the fitted functions together
with the noisy training observations.

\begin{figure}[H]
\centering
\includegraphics[width=0.85\textwidth]{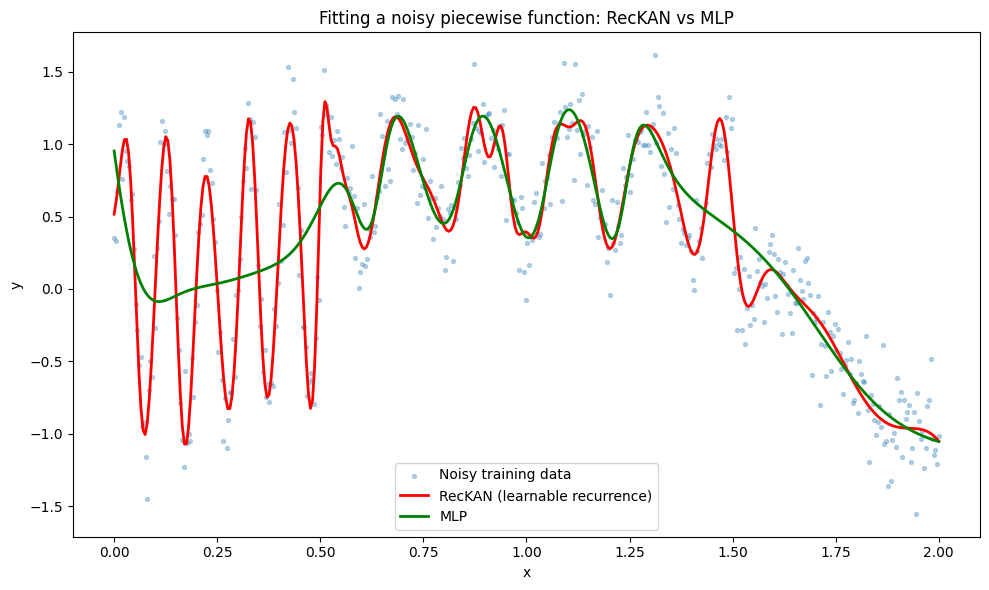}
\caption{Fitting a noisy piecewise 1D function. RecKAN (red) tracks both the
high frequency oscillatory region and the smooth region, whereas the
parameter comparable MLP (green) under fits the oscillatory portion and
approximates a local average there.}
\label{fig:piecewise-1d}
\end{figure}

\begin{table}[H]
\centering
\caption{Training loss (MSE) on the 1D piecewise function, recorded at
epoch 33{,}000.}
\label{tab:piecewise-1d}
\begin{tabular}{@{}lccc@{}}
\toprule
Model & Parameters & Epoch & MSE \\
\midrule
MLP    & 385 & 33{,}000 & 0.1583 \\
RecKAN & 149 & 33{,}000 & \textbf{0.0391} \\
\bottomrule
\end{tabular}
\end{table}

\noindent
At the recorded epoch 33{,}000 checkpoint, RecKAN attains an MSE of 0.0391,
compared with 0.1583 for the MLP, while using fewer parameters (149 versus
385). The visual fit in Figure~\ref{fig:piecewise-1d} is consistent with this
difference: RecKAN follows the high frequency part of the target as well as
the smooth portion of the domain, whereas the MLP does not recover the rapid
local oscillations.

\FloatBarrier
\paragraph{2D fractal like noisy function.}
The second target is a two dimensional multi-scale surface constructed from
sinusoids at several spatial frequencies with additive noise. RecKAN has 1,445
parameters, using three layers of shape $2\to8\to16\to1$ at basis order $8$
plus five shared recurrence parameters. It is compared against an MLP with
2,305 parameters and hidden-layer widths 64 and 32. Figure~\ref{fig:fractal-2d}
shows the predicted surfaces next to the noisy target.

\begin{figure}[H]
\centering
\includegraphics[width=0.95\textwidth]{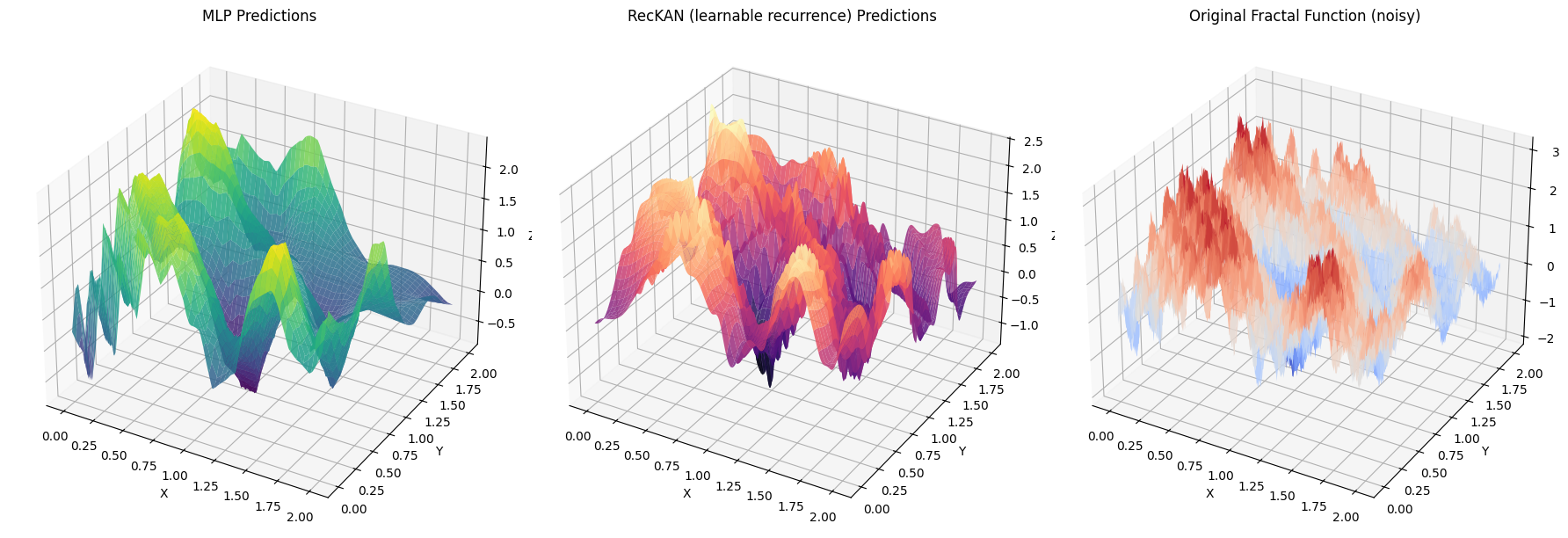}
\caption{Fitting a noisy fractal like 2D function. Left: MLP predictions.
Center: RecKAN predictions. Right: noisy target surface. Both models recover
the coarse structure; RecKAN additionally recovers more fine scale,
high frequency detail visible in the target.}
\label{fig:fractal-2d}
\end{figure}

\begin{table}[H]
\centering
\caption{Training loss (MSE) on the 2D fractal like function, recorded at
epoch 3{,}001.}
\label{tab:fractal-2d}
\begin{tabular}{@{}lccc@{}}
\toprule
Model & Parameters & Epoch & MSE \\
\midrule
MLP    & 2,305 & 3{,}001 & 0.200 \\
RecKAN & \textbf{1,445} & 3{,}001 & \textbf{0.113} \\
\bottomrule
\end{tabular}
\end{table}

\noindent
At the recorded epoch-3{,}001 checkpoint, RecKAN achieves a lower training
MSE than the MLP, 0.113 versus 0.200, while using fewer parameters. The
predicted surface in Figure~\ref{fig:fractal-2d} is consistent with the
quantitative result: both models reproduce broad variations, while RecKAN
recovers more of the target's local high frequency structure.

\FloatBarrier
\subsection{Analysis of the Learned Recurrence Parameters}
\label{subsec:exp-learned-params}

Beyond predictive performance, the five learned recurrence parameters
$(a,b,c,d,e)$ can be interpreted using the recurrence analysis in
Section~\ref{subsec:special-cases}. The two synthetic experiments were
initialized at
\begin{equation}
(a,b,c,d,e)=(0,2,0,-1,0).
\label{eq:experiment-initialization}
\end{equation}
This initialization is Chebyshev like because it has $\alpha(x)=2x$, but it
is not the exact Chebyshev setting from Proposition~\ref{prop:chebyshev}: its
far term coefficient is $\beta(x)=-x$, rather than the constant
$\beta(x)=-1$ required by the standard Chebyshev recurrence.

\begin{table}[H]
\centering
\caption{Learned recurrence parameters at the end of the synthetic runs.
Both runs use the initialization $(a,b,c,d,e)=(0,2,0,-1,0)$.}
\label{tab:learned-params}
\begin{tabular}{@{}lccccc@{}}
\toprule
Run & $a$ & $b$ & $c$ & $d$ & $e$ \\
\midrule
Initialization & 0.000 & 2.000 & 0.000 & $-1.000$ & 0.000 \\
1D piecewise fit & 0.335 & 1.789 & $-0.067$ & $-0.374$ & $-0.864$ \\
2D fractal fit & 0.049 & 2.076 & 0.001 & $-1.001$ & $-0.101$ \\
\bottomrule
\end{tabular}
\end{table}

\paragraph{2D fractal-like function.}
The 2D run remains close to its initialization. In particular,
$a=0.049$ remains near zero, while the largest change is a small constant
offset in $\beta(x)$ through $e=-0.101$. The learned basis therefore stays
close to the degree one per step regime described in
Proposition~\ref{prop:degree}. This is consistent with the visual result in
Figure~\ref{fig:fractal-2d}, where RecKAN improves reconstruction of fine
detail without requiring a large structural departure from its initial basis.

\paragraph{1D piecewise function.}
The 1D run moves more substantially in two directions. First,
$\beta(x)=dx+e$ changes from an input dependent term dominated by $d=-1$ to a
term dominated by the constant component $e=-0.864$. Second, the quadratic
coefficient grows from $a=0$ to $a=0.335$. By
Proposition~\ref{prop:degree}, this activates a recurrence whose basis
functions can accumulate up to two degrees per recursion step, rather than
one. At basis order $K=8$, the degree bound consequently expands from seven
on the $a=0$ slice to fourteen when the quadratic contribution is active.

\noindent
This provides a plausible interpretation of the fit in
Figure~\ref{fig:piecewise-1d}: the local high frequency portion of the target
may benefit from the additional polynomial capacity provided by the learned
quadratic term. The interpretation is descriptive rather than causal;
establishing causality would require controlled ablations, such as directly
comparing the learned basis spectrum with a fixed degree-8 Chebyshev basis on
the same target.

\FloatBarrier
\subsection{Benchmark Classification and Forecasting}
\label{subsec:exp-benchmarks}

We next compare RecKAN with the three KAN baselines on five real benchmark
datasets: MNIST and CIFAR-10 for image classification, AG News for text
classification, ECG5000 for biomedical time series classification, and
ETTh1 for time-series forecasting. The models are compared under the
matched-parameter protocol described in Section~\ref{subsec:exp-setup}.
For each dataset, we report the mean best test accuracy (for classification
tasks) or mean best test MSE (for the forecasting task) and standard
deviation across multiple independent runs.

\begin{table}[H]
\centering
\caption{Mean best test performance and standard deviation across multiple runs on five benchmark datasets. All models are closely parameter matched within each dataset; the largest parameter count difference is five. The best mean performance in each dataset row is shown in bold. For classification tasks, higher accuracy is better; for ETTh1 forecasting, lower MSE is better.}
\label{tab:benchmark-results}
\small
\setlength{\tabcolsep}{8pt}
\renewcommand{\arraystretch}{1.2}
\begin{tabular}{lcccccc}
\toprule
Dataset & Model & Task & $n_{\text{params}}$ & Mean Best & Std \\
\midrule
MNIST & RecKAN & classification & 96,575 & \textbf{97.393\%} & 0.1617 \\
MNIST & ChebyKAN & classification & 96,570 & 97.030\% & 0.2107 \\
MNIST & JacobiKAN & classification & 96,572 & 97.190\% & 0.1334 \\
MNIST & SplineKAN & classification & 96,570 & 96.877\% & 0.1554 \\
\midrule
CIFAR-10 & RecKAN & classification & 796,101 & \textbf{54.387\%} & 0.2285 \\
CIFAR-10 & ChebyKAN & classification & 796,096 & 53.493\% & 0.4051 \\
CIFAR-10 & JacobiKAN & classification & 796,098 & 53.203\% & 0.0900 \\
CIFAR-10 & SplineKAN & classification & 796,096 & 51.280\% & 0.0346 \\
\midrule
AG News & RecKAN & classification & 264,901 & \textbf{88.400\%} & 0.0656 \\
AG News & ChebyKAN & classification & 264,896 & 88.273\% & 0.0873 \\
AG News & JacobiKAN & classification & 264,898 & 88.230\% & 0.2778 \\
AG News & SplineKAN & classification & 264,896 & 87.723\% & 0.1250 \\
\midrule
ECG5000 & RecKAN & classification & 20,389 & \textbf{93.890\%} & 0.2575 \\
ECG5000 & ChebyKAN & classification & 20,384 & 93.757\% & 0.2254 \\
ECG5000 & JacobiKAN & classification & 20,386 & 93.807\% & 0.0848 \\
ECG5000 & SplineKAN & classification & 20,384 & 93.747\% & 0.1908 \\
\midrule
ETTh1 & RecKAN & regression & 88,229 & \textbf{0.011172} & 0.001523 \\
ETTh1 & ChebyKAN & regression & 88,224 & 0.036781 & 0.004517 \\
ETTh1 & JacobiKAN & regression & 88,226 & 0.033786 & 0.007165 \\
ETTh1 & SplineKAN & regression & 88,224 & 0.011879 & 0.000192 \\
\bottomrule
\end{tabular}
\end{table}

\noindent
RecKAN achieves the highest mean test accuracy on all four classification
benchmark datasets in Table~\ref{tab:benchmark-results}. On MNIST, RecKAN
attains 97.39\% accuracy, outperforming the strongest baseline (JacobiKAN)
by 0.20 percentage points. On CIFAR-10, RecKAN reaches 54.39\%, surpassing
the best baseline (ChebyKAN) by 0.89 percentage points. For AG News,
RecKAN achieves 88.40\%, improving upon the best baseline (ChebyKAN) by
0.13 percentage points. On ECG5000, RecKAN obtains 93.89\%, exceeding
the best baseline (JacobiKAN) by 0.08 percentage points.

\noindent
For the ETTh1 forecasting task, RecKAN achieves the lowest mean test MSE
of 0.011172, outperforming SplineKAN (0.011879), JacobiKAN (0.033786),
and ChebyKAN (0.036781). The performance gap between RecKAN and
SplineKAN is relatively small on this task, while both significantly
outperform the other two baselines. This suggests that RecKAN and
SplineKAN are particularly well suited for time series forecasting
tasks.

\begin{table}[H]
\centering
\caption{Improvement of RecKAN over baseline KAN variants across all datasets. For classification tasks, values represent the relative gain in mean best test accuracy (\%). For ETTh1, values represent the relative reduction in MSE.}
\label{tab:improvement}
\small
\begin{tabular}{@{}lcccc@{}}
\toprule
Dataset & vs ChebyKAN & vs JacobiKAN & vs SplineKAN \\
\midrule
MNIST & +0.363\% & +0.203\% & +0.516\% \\
CIFAR-10 & +0.894\% & +1.184\% & +3.107\% \\
AG News & +0.127\% & +0.170\% & +0.677\% \\
ECG5000 & +0.133\% & +0.083\% & +0.143\% \\
ETTh1 & +69.63\% & +66.93\% & +5.95\% \\
\midrule
\textbf{Average (Classification)} & \textbf{+0.379\%} & \textbf{+0.409\%} & \textbf{+1.110\%} \\
\bottomrule
\end{tabular}
\end{table}

\noindent
The results demonstrate that RecKAN consistently achieves competitive or
superior performance across diverse data modalities, including images,
text, and time series signals. The learned recursive basis proves to be
an effective alternative to fixed polynomial bases, offering improved
representational flexibility while maintaining a compact parameter
footprint. The standard deviations indicate that RecKAN exhibits stable
performance across runs, with particularly low variance on AG News
($\pm$0.07) and competitive stability on other datasets. Notably,
RecKAN shows the largest improvement on CIFAR-10 (+3.107\% over
SplineKAN and +0.894\% over ChebyKAN), a more challenging image
classification task, while maintaining consistent gains across all
benchmarks. On average, RecKAN outperforms SplineKAN by 1.11\%,
ChebyKAN by 0.38\%, and JacobiKAN by 0.41\% on the classification
benchmarks.

\FloatBarrier
\FloatBarrier
\subsection{RecKAN Classifier Heads for Convolutional Networks}
\label{subsec:exp-cnn-heads}

We further evaluate whether the learnable recursive basis can complement a
convolutional feature extractor. We compare a standard CNN with a CNN--RecKAN
variant on Fashion-MNIST, CIFAR-10, and SVHN. For each dataset, both models use
the same convolutional backbone and differ only in their classifier head. The
standard CNN uses a two layer MLP head, whereas CNN--RecKAN replaces this head
with a recursive polynomial classifier.

The models are closely parameter matched. For Fashion-MNIST, the standard CNN
contains 146{,}823 trainable parameters and CNN--RecKAN contains 146{,}863.
For CIFAR-10 and SVHN, the corresponding counts are 295{,}111 and 295{,}151.
Thus, CNN--RecKAN uses only 40 additional trainable parameters in each
comparison. All paired models use the same input preprocessing, convolutional
backbone, batch size, optimizer, learning rate schedule, weight decay, gradient
clipping, and training epoch budget.

\begin{table}[H]
\centering
\caption{Mean best test accuracy (\%) and standard deviation across multiple runs for standard CNN and CNN--RecKAN models with identical convolutional backbones and closely matched parameter counts. The best mean accuracy in each dataset row is shown in bold.}
\label{tab:cnn-reckan-results}
\small
\setlength{\tabcolsep}{8pt}
\renewcommand{\arraystretch}{1.2}
\begin{tabular}{lcccccc}
\toprule
Dataset & Model & Task & $n_{\text{params}}$ & Mean Best (\%) & Std (\%) \\
\midrule
Fashion-MNIST & CNN--RecKAN & classification & 146,863 & \textbf{93.357} & 0.1830 \\
Fashion-MNIST & Standard CNN + MLP & classification & 146,823 & 92.840 & 0.1300 \\
\midrule
CIFAR-10 & CNN--RecKAN & classification & 295,151 & \textbf{88.503} & 0.0983 \\
CIFAR-10 & Standard CNN + MLP & classification & 295,111 & 88.117 & 0.1351 \\
\midrule
SVHN & CNN--RecKAN & classification & 295,151 & \textbf{96.150} & 0.1697 \\
SVHN & Standard CNN + MLP & classification & 295,111 & {96.070} & 0.0424 \\
\bottomrule
\end{tabular}
\end{table}

\noindent
Under the matched parameter protocol, CNN--RecKAN obtains higher mean best 
test accuracy on all three datasets: 93.357\% versus 92.840\% on Fashion-MNIST, 
88.503\% versus 88.117\% on CIFAR-10, and 96.150\% versus 96.070\% 
on SVHN. The corresponding absolute improvements are 0.517, 0.386, and 0.080 
percentage points, respectively. These results demonstrate that replacing 
the MLP classifier head with a RecKAN classifier head consistently improves 
performance across different image classification tasks, achieving gains 
with only 40 additional trainable parameters. The learned recursive basis 
effectively complements the convolutional feature extractor, providing a 
lightweight yet effective enhancement for CNN-based architectures.

\begin{table}[H]
\centering
\caption{Learned recurrence coefficients at the best CNN--RecKAN checkpoint
for each image classification dataset.}
\label{tab:cnn-learned-recurrence}
\small
\begin{tabular}{@{}lrrrrr@{}}
\toprule
Dataset & $a$ & $b$ & $c$ & $d$ & $e$ \\
\midrule
Fashion-MNIST & 1.0961 & 2.2617 & 0.2105 & $-0.0990$ & $-1.4294$ \\
CIFAR-10      & 1.0026 & 2.3224 & 0.6071 & 0.6865 & $-1.7763$ \\
SVHN          & 1.4354 & 2.3903 & 0.5558 & 0.2305 & $-1.7463$ \\
\bottomrule
\end{tabular}
\end{table}

The learned recurrence parameters in Table~\ref{tab:cnn-learned-recurrence}
are reported from one representative run; they differ across datasets. In all
three CNN experiments, the quadratic coefficient $a$ becomes nonzero, placing
the learned basis outside the degree one per step sub family containing the
classical special cases in Table~\ref{tab:special-cases}. At the same time,
the learned values retain a positive coefficient near two for the linear term
$b$ and a negative constant feedback term $e$. CIFAR-10 learns the largest
input-dependent contribution in $\beta(x)=dx+e$, with $d=0.6865$, whereas
Fashion-MNIST learns a nearly constant $\beta(x)$ with $d=-0.0990$. These
task-dependent parameter settings show that the learned recurrence does not
simply retain its initial Chebyshev like configuration. They are descriptive
observations rather than causal explanations of the accuracy differences;
establishing causality would require controlled ablations.

\noindent
Across all three datasets, CNN--RecKAN consistently outperforms the standard
CNN with MLP head, with average improvement of 0.328 percentage points using
only 40 additional parameters. The most substantial gain is observed on
Fashion-MNIST (+0.517\%), while the improvement on SVHN is marginal (+0.080\%)
and within the standard deviation. The standard deviations indicate that
CNN--RecKAN exhibits slightly higher variance on Fashion-MNIST and SVHN
compared to the standard CNN, but lower variance on CIFAR-10. Overall, these
results suggest that replacing an MLP head with a RecKAN classifier head
provides a lightweight and effective enhancement for CNN-based architectures.

\FloatBarrier
\section{Conclusion}
\label{sec:conclusion}

We introduced RecKAN, a Kolmogorov--Arnold Network in which the univariate
basis is generated by a learnable second order polynomial recurrence. Rather
than selecting a spline, Chebyshev, Jacobi, wavelet, or other basis family
before optimization, RecKAN learns the coefficients of the rule that produces
the basis itself. This construction recovers several familiar
constant coefficient polynomial recurrences as special cases, including
Chebyshev, Fibonacci, Pell, and Jacobsthal polynomial sequences, while also
allowing the learned basis to move beyond these predefined families.

The degree analysis shows that activating the quadratic recurrence coefficient
changes the maximum degree growth from one to two degrees per recursion step.
This provides a direct and interpretable capacity mechanism: at a fixed basis
order, RecKAN can construct polynomial basis elements of higher degree than
the classical $a=0$ recurrence slice. The learned coefficients further make
this mechanism observable, allowing the resulting basis to be inspected rather
than treated solely as an implicit representation.

Across the five benchmark datasets (four classification and one forecasting), 
RecKAN achieves the highest reported test accuracy on all classification 
tasks and the lowest MSE on the ETTh1 forecasting benchmark among the 
compared Chebyshev, Jacobi, and spline-based KAN variants under closely 
matched parameter budgets. On the synthetic function approximation tasks, 
RecKAN obtains lower training error than the parameter comparable MLPs and 
better captures local high frequency structure. The learned recurrence 
parameters vary substantially across these tasks, indicating that the model 
adapts its basis rather than retaining its Chebyshev like initialization.

The CNN head experiments extend this observation beyond fully KAN based
architectures. With identical convolutional feature extractors and nearly
identical parameter counts, CNN--RecKAN achieves higher reported test accuracy
than matched MLP classifier heads on Fashion-MNIST, CIFAR-10, and SVHN. Taken
together, these results suggest that learning the basis generating recurrence
is a useful alternative to selecting a fixed polynomial basis in advance.
\section{Limitations and Future Work}
\label{sec:limitations}

RecKAN currently uses one set of five recurrence coefficients shared across
all edges and layers of a network. This global sharing makes the learned basis
compact and easy to interpret, but it may limit expressiveness when different
layers, channels, or regions of an input require substantially different
univariate representations. A natural extension is to study layer specific,
group specific, or edge specific recurrences while retaining constraints that
preserve numerical stability and interpretability.

The present formulation uses a second order recurrence with a quadratic
coefficient in the current term and a linear coefficient in the preceding
term. This choice provides a small, tractable parameterization and includes
several well known polynomial sequences, but it does not cover all classical
orthogonal polynomial families whose recurrence coefficients vary with the
basis index. Extending RecKAN to controlled index dependent recurrences,
higher order recurrences, rational recurrences, or mixtures of recursive and
non-polynomial basis functions could enlarge the family of adaptive bases
available to KANs.

The stabilization scheme relies on bounded recurrence coefficients, bounded
inputs, basis normalization, and separate optimization control for recurrence
parameters. These components make the current implementation practical, but
their interactions with basis order, network depth, normalization design, and
different data modalities remain open. A more detailed analysis of the
optimization geometry and numerical properties of learned recurrences may
clarify when particular recurrence structures are preferable.

Finally, the learned coefficients provide a compact description of the basis,
but their task level interpretation remains descriptive. Future work can
connect coefficient trajectories to explicit spectral properties, study the
effective polynomial degrees used by trained edge functions, and develop
regularizers that encourage selected structural properties such as sparsity,
proximity to known polynomial families, or controlled departures from a
reference recurrence. These directions may further strengthen RecKAN as both
a flexible function approximator and an interpretable basis learning
framework.

% =====================================================================
%  APPENDIX: ADDITIONAL PLOTS
% =====================================================================
\appendix
\section{Additional Visualizations of the Recursive Basis}
\label{app:extra-plots}

In this appendix, we provide comprehensive visualizations of the recursive polynomial basis functions \(R_n(x)\) and the learned recurrence coefficients across all experiments discussed in Section~\ref{subsec:exp-learned-params}. These plots supplement the main text by offering a detailed view of how the model adapts its basis to different tasks.

\begin{figure}[H]
\centering
\includegraphics[width=0.95\textwidth]{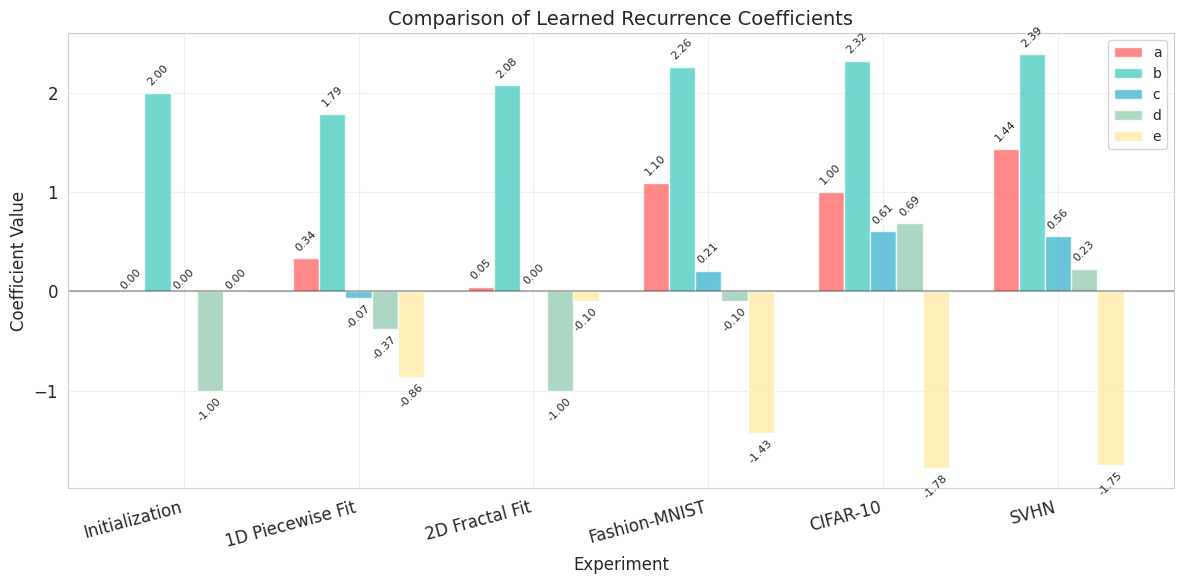}
\caption{Comparison of learned recurrence coefficients across different experiments. The bar chart illustrates the distinct values of \(a, b, c, d, e\) learned by the model, showing significant deviation from the initial Chebyshev like configuration, particularly in image-based benchmarks.}
\label{fig:app_coeff_comparison}
\end{figure}

\begin{figure}[H]
\centering
\includegraphics[width=0.95\textwidth]{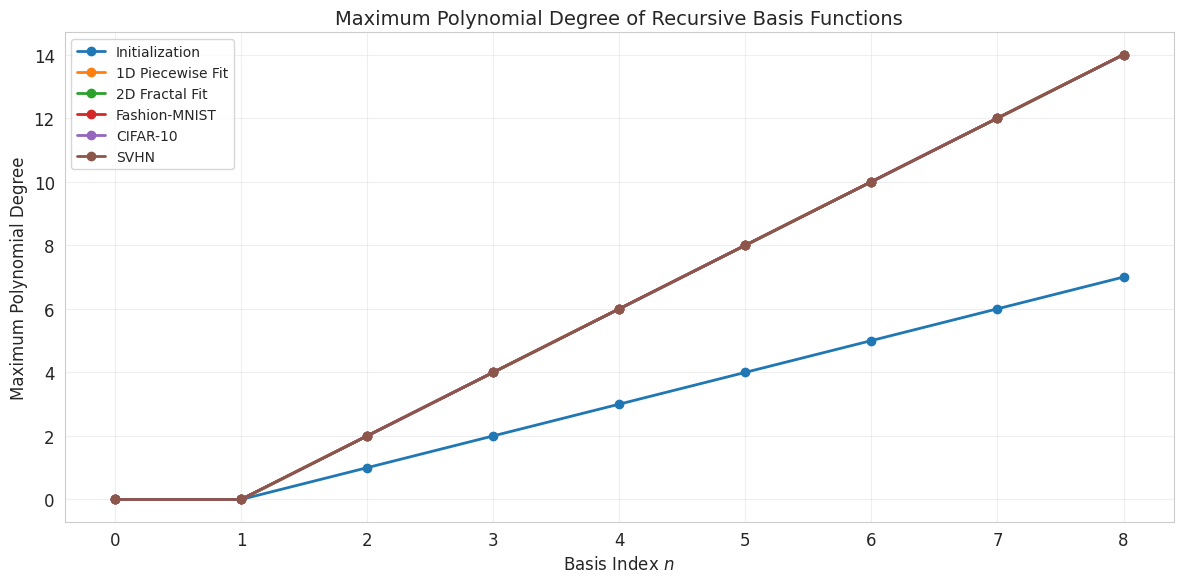}
\caption{Maximum polynomial degree of recursive basis functions as a function of the basis index \(n\). The plot confirms that the learned coefficients (e.g., in the SVHN experiment) allow the basis to reach higher polynomial degrees (up to 14) compared to the initialization (up to 7).}
\label{fig:app_degree_comparison}
\end{figure}

\begin{figure}[H]
\centering
\includegraphics[width=0.95\textwidth]{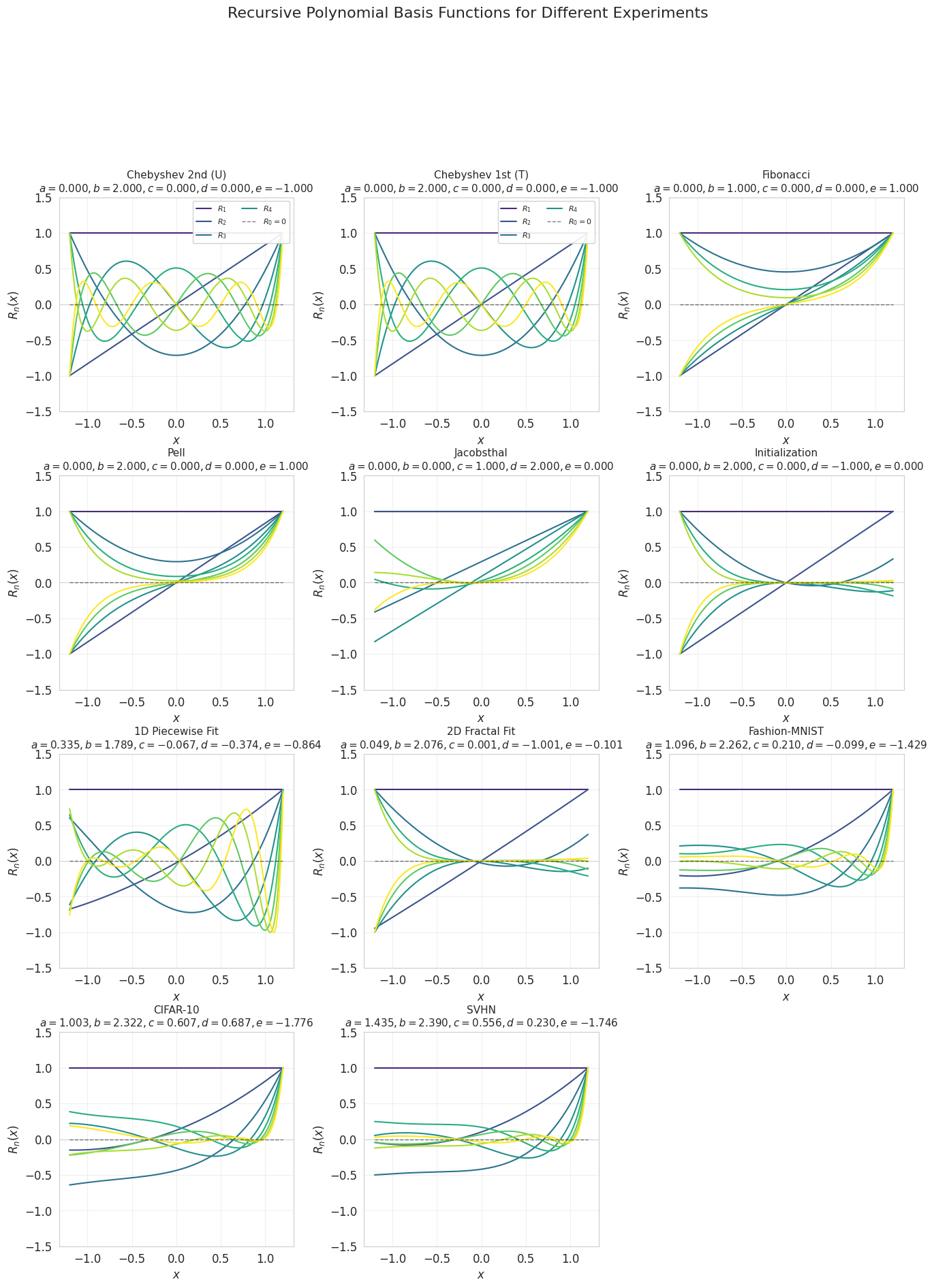}
\caption{Comparison of the recursive polynomial basis functions, including classical polynomial families (Chebyshev 1st/2nd kind, Fibonacci, Pell, Jacobsthal) and the learned bases for different experiments. This plot highlights that RecKAN can recover known families or move beyond them depending on the learned coefficients.}
\label{fig:app_basis_all}
\end{figure}

\section{Complete Proof of Proposition 1 (Degree Growth)}
\label{app:proof-degree}

Degree growth:
For \(a \neq 0\), \(\deg R_n \leq 2(n-1)\) for \(n \geq 1\). For \(a = 0\),
\(\deg R_n \leq n-1\). Equality holds in either case whenever no leading term
cancellation occurs.

\begin{proof}
Let \(\delta_n = \deg R_n\) for \(n \geq 0\), with the convention that \(\deg 0 = -\infty\).
From Definition~\ref{def:recurrence}, we have:
\begin{align}
R_0(x) &= 0 \quad \Rightarrow \quad \delta_0 = -\infty, \\
R_1(x) &= 1 \quad \Rightarrow \quad \delta_1 = 0, \\
R_{n+1}(x) &= \alpha(x)R_n(x) + \beta(x)R_{n-1}(x), \quad n \geq 1,
\end{align}
where \(\alpha(x) = ax^2 + bx + c\) and \(\beta(x) = dx + e\).

\paragraph{Case 1: \(a \neq 0\).}
In this case, \(\deg \alpha = 2\) and \(\deg \beta \leq 1\). For \(n \geq 1\), we have:
\begin{equation}
\deg(\alpha R_n) = 2 + \delta_n, \quad \deg(\beta R_{n-1}) \leq 1 + \delta_{n-1}.
\end{equation}
Therefore, by the triangle inequality for polynomial degrees:
\begin{equation}
\delta_{n+1} \leq \max\{2 + \delta_n,\; 1 + \delta_{n-1}\}.
\end{equation}

We prove by induction that \(\delta_n \leq 2(n-1)\) for all \(n \geq 1\).

\emph{Base case:} For \(n = 1\), \(\delta_1 = 0 \leq 2(1-1) = 0\), so the bound holds.

\emph{Inductive step:} Assume \(\delta_k \leq 2(k-1)\) for all \(k \leq n\). Then:
\begin{align}
\delta_{n+1} &\leq \max\{2 + \delta_n,\; 1 + \delta_{n-1}\} \\
&\leq \max\{2 + 2(n-1),\; 1 + 2(n-2)\} \\
&= \max\{2n,\; 2n-3\} \\
&= 2n = 2((n+1)-1).
\end{align}
Thus, by induction, \(\delta_n \leq 2(n-1)\) for all \(n \geq 1\).

\paragraph{Case 2: \(a = 0\).}
In this case, \(\deg \alpha \leq 1\) (since \(\alpha(x) = bx + c\)) and \(\deg \beta \leq 1\). For \(n \geq 1\), we have:
\begin{equation}
\deg(\alpha R_n) \leq 1 + \delta_n, \quad \deg(\beta R_{n-1}) \leq 1 + \delta_{n-1}.
\end{equation}
Therefore:
\begin{equation}
\delta_{n+1} \leq 1 + \max\{\delta_n,\; \delta_{n-1}\}.
\end{equation}

We prove by induction that \(\delta_n \leq n-1\) for all \(n \geq 1\).

\emph{Base case:} For \(n = 1\), \(\delta_1 = 0 \leq 1-1 = 0\), so the bound holds.

\emph{Inductive step:} Assume \(\delta_k \leq k-1\) for all \(k \leq n\). Then:
\begin{align}
\delta_{n+1} &\leq 1 + \max\{\delta_n,\; \delta_{n-1}\} \\
&\leq 1 + \max\{n-1,\; n-2\} \\
&= 1 + (n-1) = n = (n+1)-1.
\end{align}
Thus, by induction, \(\delta_n \leq n-1\) for all \(n \geq 1\).

\paragraph{Equality conditions.}
Equality holds when no cancellation occurs between the dominant terms of \(\alpha(x)R_n(x)\) and \(\beta(x)R_{n-1}(x)\). In the \(a \neq 0\) case, this requires:
\begin{equation}
a \cdot \text{lc}(R_n) + d \cdot \text{lc}(R_{n-1}) \neq 0
\end{equation}
for each \(n\), where \(\text{lc}(R_n)\) denotes the leading coefficient of \(R_n\). In the \(a = 0\) case, equality requires:
\begin{equation}
b \cdot \text{lc}(R_n) + d \cdot \text{lc}(R_{n-1}) \neq 0.
\end{equation}

\paragraph{Interpretation.}
The bound \(\delta_n \leq 2(n-1)\) for \(a \neq 0\) means that the basis functions can grow by up to two degrees per recursion step, while for \(a = 0\), they grow by at most one degree per step. This is because the quadratic term \(ax^2\) in \(\alpha(x)\) doubles the degree contribution at each step. For example, at basis order \(K = 8\):
\begin{itemize}
\item If \(a = 0\), the maximum achievable degree is \(7\).
\item If \(a \neq 0\), the maximum achievable degree is \(14\).
\end{itemize}
This demonstrates that by activating the quadratic term, we can construct basis functions with higher polynomial degrees without increasing the basis order \(K\).
\end{proof}

\end{document}